\documentclass[12pt,reqno]{amsart}
\usepackage{amsfonts}
\usepackage{amsmath}
\usepackage{amsthm}
\usepackage{mathrsfs}
\usepackage{amssymb}
\usepackage{stmaryrd}
\usepackage{graphicx}
\usepackage{epstopdf}
\usepackage{textcomp}
\usepackage[margin=1in]{geometry}

\newtheorem{theorem}{Theorem}

\theoremstyle{definition}

\begin{document}
\title[PDE approach to the problem of online prediction]{PDE approach to the problem of online prediction with expert advice: a construction of potential-based strategies}

\author{Dmitry B. Rokhlin}

\address{Institute of Mathematics, Mechanics and Computer Sciences,
              Southern Federal University,
Mil'chakova str., 8a, 344090, Rostov-on-Don, Russia}
\email[Dmitry B. Rokhlin]{rokhlin@math.rsu.ru}


\begin{abstract}
We consider a sequence of repeated prediction games and formally pass to the limit. The supersolutions of the resulting non-linear parabolic partial differential equation are closely related to the potential functions in the sense of N.\,Cesa-Bianci, G.\,Lugosi (2003). Any such supersolution gives an upper bound for forecaster's regret and suggests a potential-based prediction strategy, satisfying the Blackwell condition. A conventional upper bound for the worst-case regret is justified by a simple verification argument. 
\end{abstract}

\maketitle

\section{Introduction}
\label{sec:1}
Let $B$ be any set. In the problem of online prediction with expert advice a forecaster predicts a sequence $(b_t)_{t=0}^{n-1}$, $b_t\in B$ on the basis of expert opinions $f_t^i\in A$, $i=1,\dots,N$, where $A$ is a convex subset of a vector space. More precisely, at round $t\in\{0,\dots,n-1\}$ forecaster's guess $a_t$ is a convex combination of expert advices:
$$a_t=\langle p_t,f_t\rangle:=\sum_{i=1}^N p_t^i f^i_t,\qquad p_t\in\Delta:=\left\{z\ge 0:\langle z,1\rangle=1\right\},$$
based on the available history and current advices: $p_t=p_t((b_s)_{s=0}^{t-1}$, $(f_s)_{s=0}^t)$.

Let $l:A\times B\mapsto [0,1]$ be a loss function. Forecaster's aim is to keep the \emph{regret}
$$ R_n=\sum_{t=0}^{n-1}  l(\langle p_t,f_t\rangle,b_t)-\min_{1\le i\le N}\sum_{t=0}^{n-1} l(f^i_t,b_t) $$
small. This regret $R_n$ measures the quality of predictions by comparing the cumulative loss of the forecaster with that of a best expert, chosen in hindsight. 

We refer to \cite{CesLug06} for more information on this problem. The basic result (see, e.g, \cite[Theorem 2.2]{CesLug06}) guarantees the existence of a prediction strategy $p^*$ achieving the uniform bound
\begin{equation} \label{eq:1}
\frac{R_n\left((p_t^*)_{t=0}^{n-1},(f_t)_{t=0}^{n-1},(b_t)_{t=0}^{n-1}\right)}{\sqrt n}\le C
\end{equation}
for any $(b_t)_{t=0}^{n-1}$, $(f_t)_{t=0}^{n-1}$ under the assumption that $l$ is convex in its first argument.
Moreover, this bound cannot be improved without further assumptions: \cite[Theorem 3.7]{CesLug06}. The inequality (\ref{eq:1}) implies that in the long run on average the forecaster predicts as well as a best expert: $R_n/n\to 0$, $n\to\infty$. 

There are plenty of strategies achieving the bound (\ref{eq:1}). In \cite{CesLug03} it was shown that for a rather general class of online learning problems the construction of such strategies can be based on the notion of potential function. More recently \cite{RakShaSri12} proposed a systematic way for the construction of potentials in the case of randomized prediction, mentioning that ``The origin/recipe for ``good'' potential functions has always been a mystery (at least to the authors).'' The authors of \cite{RakShaSri12} considered a recurrence relation, for the value function of a repeated game, determining the optimal regret, and showed that potential functions are related to relaxations of this function, which are consistent with the mentioned recurrence relation. To obtain such relaxations they used upper bounds, developed in the theory of online learning and capturing the complexity of the problem.

In this paper we show that for the problem of prediction with expert advice there is another ``natural'' way for the appearance of potential-based algorithms. As in \cite{RakShaSri12}, we consider a repeated game, determining the optimal regret, and the correspondent recurrence relation for the value functions $v^n$. Further, in contrast to \cite{RakShaSri12}, we simply pass to the limit as $n\to\infty$ and get a non-linear parabolic Bellman-Isaacs type partial differential equation in $[0,1]\times\mathbb R^N$. A rigorous justification of this procedure can be performed within the theory of viscosity solutions. However, being interested only in the construction of prediction strategies, we need not do it! As usual, a Bellman-type equation at least formally produces optimal strategies. More precisely, we consider the strategies, generated by appropriate smooth supersolutions, and then directly check the inequality (\ref{eq:1}), using the argumentation similar to that of the verification method from the theory of optimal control.  

The described approach is mainly inspired by the paper \cite{KohSer10}, where there was studied a link between fully non-linear second order (parabolic and elliptic) PDE and repeated games. Its application to the problems of online learning theory was initiated in \cite{Rok17}, where an asymptotics of the sequential  Rademacher complexity (the last notion was introduced in \cite{RakSriTew10}) of a finite function class was related to the viscosity solution of a $G$-heat equation. 
In turn, the result of \cite{Rok17} is based on the central limit theorem under model uncertainty, studied within the same approach in \cite{Rok15a}. 

\section{Prediction game and the limiting PDE}
\label{sec:2}
The worst-case regret
$$ \mathscr R_n=\sup_{f_0\in A^N}\inf_{p_0\in\Delta}\sup_{b_0\in B}\dots\sup_{f_{n-1}\in A^N}\inf_{p_{n-1}\in\Delta}\sup_{b_{n-1}\in B} R_n((p_t)_{t=0}^{n-1},(f_t)_{t=0}^{n-1},(b_t)_{t=0}^{n-1})$$
is a result of the repeated game between the predictor, an adversary and experts. In this game the adversary has an informational advantage over the predictor and experts, since $b_t$ is chosen after the sequences $(p_j)_{j=0}^t$, $(f_j)_{j=0}^t$ are revealed. Furthermore, the predictor has an informational advantage over the experts, since the choice of $p_t$ can be based on $(f_j)_{j=0}^t$, $(b_j)_{j=0}^{t-1}$. Finally, experts can use only the information contained in $(p_j)_{j=0}^{t-1}$, $(b_j)_{j=0}^{t-1}$. The adversary and experts play against the predictor, trying to maximize his regret.

To get a recurrent formula for $\mathscr R_n$ let us introduce the family of state processes
\begin{align} \label{eq:1A}
X_{s+1}^{i,t,x,p,f,b} &=X_s^{i,t,x,p,f,b}+\frac{1}{\sqrt n} r^i(p_s,f_s,b_s),\quad s=t,\dots,n-1,\\
 r^i &=l(\langle p_t,f_t\rangle,b_t)-l(f^i_t,b_t),\quad X_t^{i,t,x,p,f,b}=x^i\in\mathbb R.\nonumber
\end{align} 
Summing up the increments $X_{s+1}^{i,0,0,p,f,b}-X_s^{i,0,0,p,f,b}$, we obtain
$$ \frac{1}{\sqrt n}R_n((p_t)_{t=0}^{n-1},(f_t)_{t=0}^{n-1},(b_t)_{t=0}^{n-1})=\max\left\{X_n^{1,0,0,p,f,b},\dots,X_n^{N,0,0,p,f,b}\right\}.$$

Let us introduce the value functions
$$ v^n(t/n,x)=\sup_{f_t\in A^N}\inf_{p_t\in \Delta}\sup_{b_t\in B}\dots\sup_{f_{n-1}\in A^N}\inf_{p_{n-1}\in \Delta}\sup_{b_{n-1}\in B} g(X_n^{t,x,p,f,b}),$$
$$ v^n(1,x)=g(x),$$
where $t=0,\dots,n-1$, $g(x)=\max\{x_1,\dots,x_n\}$.
From the dynamic programming theory it is known that $v^n$ satisfies the recurrence relations
\begin{equation} \label{eq:2}
v^n(t/n,x)=\sup_{f\in A^N}\inf_{p\in\Delta}\sup_{b\in B} v^n((t+1)/n,x+r(p,f,b)/\sqrt n),
\end{equation}
$t\le n-1$, $r=(r^1,\dots, r^N)$. We stress that we need not rigorously justify this and subsequent claims, since 
our goal is to formally construct prediction strategies. Their verification is delayed to the last step.

For a moment imagine that $v^n$ is a smooth function, satisfying (\ref{eq:2}) on $[0,1-1/n]\times\mathbb R^N$. Then, by Taylor's formula we get
\begin{align} \label{eq:3}
 0&=\sup_{f\in A^N}\inf_{p\in\Delta}\sup_{b\in B}\Bigl\{\sqrt{n} \langle v_x^n(t,x),r(p,f,b)\rangle+v_t^n(t,x)\nonumber\\
  &+\frac{1}{2}\langle v^n_{xx}(t,x) r(p,f,b),r(p,f,b)\rangle+o(1)\Bigr\},
\end{align} 
where $v^n_x$, $v^n_{xx}$ are the gradient vector and the Hessian matrix.

We will say that the loss function $l$ satisfies the \emph{Blackwell condition} if
\begin{equation} \label{eq:3A}
\Gamma(\gamma,f):=\{p\in\Delta:\langle\gamma,r(p,f,b)\rangle\le 0,\ b\in B\}\neq\emptyset
\end{equation} 
for all $(\gamma,f)\in\mathbb R^N_+\times A^N$.
Clearly, $\Gamma(0,f)=\Delta$. The Blackwell condition (\ref{eq:3A}) is satisfied if $l$ is convex in its first argument. In this case $p=\gamma/\langle \gamma,1\rangle\in\Gamma(\gamma,f)$ for $\gamma\in\mathbb R^N_+\backslash\{0\}$, since
$$ \sum_{i=1}^N\gamma^i r^i\left(\frac{\gamma}{\langle \gamma,1\rangle},f,b\right)=\langle \gamma,1\rangle l\left(\frac{\langle\gamma,f\rangle}{\langle \gamma,1\rangle},b\right)-\sum_{i=1}^N\gamma^i l(f^i,b)\le 0$$
by Jensen's inequality.

By the nature of $v^n$ these functions are non-decreasing in each $x_i$. Indeed, $v^n(t/n,x)$ is the optimal worst-case regret if the initial regret with respect to $i$-th expert at time moment $t$ equals to $x_i$. From (\ref{eq:3}) we get
$$ 0\le\sup_{f\in A^N}\inf_{p\in \Gamma(v_x^n,f)}\sup_{b\in B}\left\{v_t^n(t,x)+\frac{1}{2}\langle v^n_{xx}(t,x) r(p,f,b),r(p,f,b)\rangle+o(1)\right\}.$$
So, we expect that the limiting function $v$ satisfies the inequality
\begin{align} \label{eq:3B}
&-v_t(t,x)-G(v_x(t,x),v_{xx}(t,x))\le 0,\\
&G(\gamma,S)=\frac{1}{2}\sup_{f\in A^N}\inf_{p\in \Gamma(\gamma,f)}\sup_{b\in B}\langle S r(p,f,b),r(p,f,b)\rangle,\nonumber
\end{align}
and the boundary condition $v(1,x)=g(x)$.
Note that $G(\gamma,S)\ge G(\gamma,S')$, if the symmetric $N\times N$ matrix $S-S'$ is non-negative definite. Hence,
\begin{equation} \label{eq:4}
 -u_t(t,x)-G(u_x(t,x),u_{xx}(t,x))=0
\end{equation} 
is a fully non-linear parabolic equation (see \cite{CraIshLio92}). 
Along with (\ref{eq:4}) we consider the boundary condition
\begin{equation} \label{eq:5}
 u(1,x)=g(x)=\max\{x_1,\dots,x_n\}.
\end{equation}

The functions $v^n$ are defined on $Q^n=\{0,1/n,\dots,(n-1)/n,1\}\times \mathbb R^N.$
To describe their limiting behavior in a rigorous way, one can consider the Barles-Perthame half-relaxed (weak) upper limit:
\begin{align*}
\overline v(t,x)=\sup\{&\lim v^{n_k}(t_k,x_k):Q^{n_k}\ni (t_k,x_k)\to (t,x)\\
&\textrm{and}\ v^{n_k}(t_k,x_k)\ \textrm{converges}\}.
\end{align*}
From the results of \cite{BarPer88,BarSou91,KohSer10} and the above calculations we expect that $\overline v$ is a  \emph{viscosity subsolution} of (\ref{eq:4}), (\ref{eq:5}).
Note, that by the definition, 
\begin{equation} \label{eq:6}
\limsup_{n\to\infty}\frac{1}{\sqrt n}\mathscr R_n\le \overline v(0,0).
\end{equation}
 
\section{Smooth supersolutions and induced weighted average forecasting strategies}
\label{sec:3}
Take a smooth supersolution $w$ of (\ref{eq:4}), (\ref{eq:5}):
\begin{equation} \label{eq:6A}
 -w_t(t,x)-G(w_x(t,x),w_{xx}(t,x))\ge 0, \quad w(1,x)\ge g(x),
\end{equation} 
which is non-decreasing in each variable $x_i$. Assuming a comparison result: $\overline v\le w$, we conclude that the inequality (\ref{eq:6}) holds true for $w(0,0)$ instead of $\overline v(0,0)$. We also expect that a strategy $p_t(x)\in \Gamma(w_x(t,x),f)$ will produce the regret, satisfying this bound.

Let us look for supersolutions of the form $w(t,x)=c(1-t)+\Phi(x)$, where $c$ is a constant,
\begin{equation} \label{eq:7}
\Phi(x)\ge g(x),
\end{equation}
and $\Phi$ is non-decreasing in each variable. The differential inequality (\ref{eq:6A}) implies the condition
$G(\Phi_x(x),\Phi_{xx}(x))\le c$. This condition is satisfied if
\begin{equation} \label{eq:8}
 \frac{1}{2} \sup_{x\in\mathbb R^N}\langle\Phi_{xx}(x) h,h\rangle\le c,\quad |h^i|\le 1,\ i=1,\dots,N.
\end{equation}

Then by the Blackwell condition (\ref{eq:3A}) there exists a vector-function 
\begin{equation} \label{eq:9}
p^*(x,f)\in\Gamma(\Phi_x(x),f).
\end{equation}
If $l$ is convex in its first argument and $\Phi$ is strictly increasing in each variable, then, according to the remark after the formula (\ref{eq:3A}), one can take 
\begin{equation} \label{eq:10}
 p^*(x,f)=\frac{\Phi_x(x)}{\langle\Phi_x(x),1\rangle}\in \Gamma(\Phi_x(x),f).
\end{equation}

Consider the discrete-time state process (\ref{eq:1A}), generated by the prediction strategy, related to $p^*$:
\begin{equation} \label{eq:11}
 X^{*,i}_{s+1}=X_s^{*,i}+\frac{1}{\sqrt n} r^i(p_s^*,f_s,b_s),\quad p_s^*=p^*(X_s^*,f_s),\quad X^*_0=0.
\end{equation}
Note, that $p_t^*$ automatically satisfies the inequality 
\begin{equation} \label{eq:12}
\langle\Phi_x(X_t^*),r(p_t^*,f_t,b_t)\rangle\le 0
\end{equation}
which is also called the Blackwell condition: see \cite{CesLug03,CesLug06}.
For a convex function $a\mapsto l(a,b)$ from (\ref{eq:10}) we get a weighted average forecaster:
\begin{equation} \label{eq:12A}
a_t^*=\langle p_t^*,f_t\rangle=\frac{\langle \Phi_x(X_t^*),f_t\rangle}{\langle \Phi_x(X_t^*),1\rangle}.
\end{equation}

\begin{theorem} \label{th:1}
Let the Blackwell condition (\ref{eq:3A}) be satisfied, and let $\Phi:\mathbb R^N\mapsto\mathbb R$ be a twice continuously differentiable function, which non-decreases in each variable and meets the conditions (\ref{eq:7}), (\ref{eq:8}). Then a prediction strategy 
$$a_t^*=\langle p_t^*,f\rangle,\quad p_t^*=p^*(X^*_t,f_t),\quad t=0,\dots,n-1,$$ 
where $p^*$ satisfies (\ref{eq:9}) and $X^*$ is defined by (\ref{eq:11}), produces the regret, satisfying the inequality (\ref{eq:1}) with 
$C=c+\Phi(0).$ 
\end{theorem}
\begin{proof}
For $w(t,x)=c(1-t)+\Phi(x)$ by Taylor's formula we get
\begin{align*}
w((t+1)/n,X^*_{t+1})&-w(t/n,X_t^*)=-c/n+\Phi(X^*_{t+1})-\Phi(X^*_t)\\
&=-c/n+\langle\Phi_x(X_t^*),r(p_t^*,f_t,b_t)\rangle/\sqrt n\\
&+\frac{1}{2}\langle\Phi_{xx}(\xi_t)r(p_t^*,f_t,b_t),r(p_t^*,f_t,b_t^*)\rangle/n\le 0,
\end{align*}
for some $\xi_t$, where the last inequality is implied by (\ref{eq:12}) and (\ref{eq:8}). Now the assertion of the theorem follows from the condition (\ref{eq:7}):
$$ R_n/\sqrt n=g(X_n^*)\le w(1,X_n^*)=\Phi(X_n^*)\le w(0,X_0^*)=c+\Phi(0).\qedhere$$
\end{proof}

Following \cite{CesLug03,CesLug06} we call $\Phi$ a potential function. The most natural smooth upper bound for $\max\{x_1,\dots,x_N\}$, and hence a candidate for a potential, is a soft-maximum function
\begin{equation} \label{eq:13}
 \Phi(x)=\frac{1}{\eta}\ln\left(\sum_{i=1}^N e^{\eta x_i}\right),\quad \eta>0.
\end{equation}  
This function is included in a more general class $\Phi(x)=\psi\left(\sum_{i=1}^N \phi(x_i)\right)$ considered in  \cite{CesLug03,CesLug06},
where $\psi$ and $\phi$ are assumed to be concave and convex respectively. The following inequality is also taken from  \cite{CesLug03,CesLug06}:
$$\langle\Phi_{xx}(x)h,h\rangle
\le \psi'\left(\sum_{k=1}^N\phi(x_k)\right)\sum_{i=1}^N\phi''(x_i) h_i^2.$$
For (\ref{eq:13}) we have $\psi(x)=\eta^{-1}\ln x$, $\phi(x)=e^{\eta x}$,
$$ \frac{1}{2}\langle\Phi_{xx}(x)h,h\rangle\le\frac{\eta}{2}\sum_{i=1}^N\frac{e^{\eta x_i}}{\sum_{k=1}^N e^{\eta x_k}}h_i^2\le  c=\frac{\eta}{2},\quad |h_i|\le 1.$$
For $p_t^*$ generated by (\ref{eq:13}), in accordance with Theorem \ref{th:1} we have
$$ \frac{R_n}{\sqrt n}\le c+\Phi(0)=\frac{\eta}{2}+\frac{\ln N}{\eta}=\sqrt{2\ln N}$$
for an ``optimal'' choice $\eta=\sqrt{2\ln N}$ (cf. \cite[Corollary 2.2]{CesLug06}). 
The formula (\ref{eq:12A}) reduces to
$$a_t^*=\frac{\langle \Phi_x(X_t^*),f_t\rangle}{\langle \Phi_x(X_t^*),1\rangle}
=\sum_{i=1}^N\frac{e^{\eta X_t^{*,i}}}{\sum_{j=1}^N e^{\eta X_t^{*,j}}}f^i_t
=\sum_{i=1}^N\frac{e^{-\eta L_t^i/\sqrt n}}{\sum_{j=1}^N e^{-\eta L_t^j/\sqrt n}}f^i_t,$$
where $L_t^i=\sum_{s=0}^{t-1} l(f^i_t,b_s)$ is the cumulative loss of $i$-th expert. This is a
basic version of the exponentially weighted average forecaster: see \cite[Chapter 2]{CesLug06}.

\section{Randomized prediction}
\label{sec:4}
Assume that the forecaster randomly chooses a prediction by taking a sample $I_t$ from a probability distribution $p_t=(p_t^1,\dots,p_t^N)$ over $\{y_1,\dots,y_N\}$. His cumulative loss is compared with the cumulative loss of a best fixed prediction: 
$$ \sum_{t=0}^{n-1} l(I_t,b_t)-\min_{1\le i\le N} \sum_{t=0}^{n-1} l(y_i,b),$$
and the regret is defined as the expectation of this quantity with respect to the induced artificial probability measure:
$$ R_n=\sum_{t=0}^{n-1} \sum_{j=1}^N p_t^j l(y_j,b_t)-\min_{1\le i\le N} \sum_{t=0}^{n-1} l(y_i,b_t)=\min_{1\le i\le N}\sum_{t=0}^{n-1} r^i(p_t,b_t),$$
$$r^i(p,b)=\sum_{j=1}^N p^j l(y_j,b)-l(y_i,b).$$
The game, where the forecaster knows the previous moves: $p_t=p_t(b_0,\dots,b_{t-1})$, and the adversary knows the prediction algorithm: $b_t=b_t(p_0,\dots,p_t)$ but not the predictions $I_t$ itself, corresponds to the case of an oblivious adversary: \cite[Chapter 2]{CesLug06}. However, the case of non-oblivious adversary is not interesting for the problem of this form: see \cite[Lemma 4.1]{CesLug06}. 

The described game is simpler than that considered above, since the ``experts'', corresponding to fixed predictions, do not play against the forecaster. Moreover, the condition (\ref{eq:3A}) is satisfied regardless of the convexity of $l$. Repeating the reasoning of Section \ref{sec:2}, we get the inequality (\ref{eq:3B}) with
$$G(\gamma,S)=\frac{1}{2}\inf_{p\in \Gamma(\gamma)}\sup_{b\in B}\langle S r(p,b),r(p,b)\rangle,$$
$$ \Gamma(\gamma)= \left\{p\in\Delta:\sum_{i=1}^N\gamma_i\sum_{j=1}^N p^j l(y_j,b)-\sum_{i=1}^N l(y_i,b)\le 0,\ b\in B\right\}.$$
So, a prediction strategy satisfying
$$ p_t^{i,*}=\frac{\Phi_{x^i}(X_t^*)}{\langle \Phi_x(X_t^*),1\rangle}\qquad \textrm{for\ \ } \Phi_x(X_t^*)\neq 0,$$
where $\Phi$ meets the conditions of Theorem 1, and $X_t^*$ is defined by the recursion of the form (\ref{eq:11}), produces the regret $R_n\le C/\sqrt n$. In particular, $C=\sqrt{2\ln N}$ for the exponentially weighted average forecaster, discussed after Theorem \ref{th:1}.

Finally, we note that the case of internal regret (see \cite[Section 4.4]{CesLug06}) can be considered in the same way.

\section{Acknowledgments}
The research is supported by the Russian Science Foundation, project 17-19-01038. 

\bibliographystyle{spmpsci}
\bibliography{litExperts}
\end{document}